\documentclass[conference]{IEEEtran}
\IEEEoverridecommandlockouts
\usepackage{mathtools, amsthm, amsmath,amssymb,amsfonts, dsfont}
\usepackage[caption=false,font=footnotesize,labelfont=rm,textfont=rm]{subfig}
\usepackage{algorithmic}
\usepackage{graphicx}
\usepackage{textcomp}
\usepackage{cite}
\usepackage{algorithmic}
\usepackage{algorithm}
\usepackage{xspace}
\usepackage{cuted} 
\usepackage{xcolor}

\newtheorem{theorem}{Theorem}
\newtheorem{definition}{Definition}

\newtheorem{lemma}{Lemma}

\DeclareMathOperator*{\argmin}{arg\,min}

\DeclarePairedDelimiter\abs{\lvert}{\rvert}

\newcommand{\cD}{\mathcal{D}}

\newcommand{\vw}{\boldsymbol{w}}

\newcommand{\vh}{\boldsymbol{h}}

\begin{document}

\title{The Effect of Quantization in Federated Learning: A R\'enyi Differential Privacy Perspective
}

\author{\IEEEauthorblockN{Tianqu Kang, Lumin Liu, Hengtao He, Jun Zhang, S. H. Song, and Khaled B. Letaief}\\
\IEEEauthorblockA{
Dept. of ECE, The Hong Kong University of Science and Technology, Hong Kong \\
Email: tkang@connect.ust.hk, lliubb@ust.hk, eehthe@ust.hk, eejzhang@ust.hk, eeshsong@ust.hk, eekhaled@ust.hk
}}

\maketitle

\begin{abstract}

Federated Learning (FL) is an emerging paradigm that holds great promise for privacy-preserving machine learning using distributed data. To enhance privacy, FL can be combined with Differential Privacy (DP), which involves adding Gaussian noise to the model weights. However, FL faces a significant challenge in terms of large communication overhead when transmitting these model weights. To address this issue, quantization is commonly employed. Nevertheless, the presence of quantized Gaussian noise introduces complexities in understanding privacy protection.  This research paper investigates the impact of quantization on privacy in FL systems. We examine the privacy guarantees of quantized Gaussian mechanisms using Rényi Differential Privacy (RDP). By deriving the privacy budget of quantized Gaussian mechanisms, we demonstrate that lower quantization bit levels provide improved privacy protection. To validate our theoretical findings, we employ Membership Inference Attacks (MIA), which gauge the accuracy of privacy leakage. The numerical results align with our theoretical analysis, confirming that quantization can indeed enhance privacy protection. This study not only enhances our understanding of the correlation between privacy and communication in FL but also underscores the advantages of quantization in preserving privacy.

\end{abstract}


\section{Introduction} \label{intro}

The success of deep learning is highly dependent on the availability of data, and there is a growing need to utilize private data to train deep learning models. Federated Learning (FL) is a privacy-preserving learning framework, where clients can collaboratively train a model without sharing their data \cite{mcmahan2017communication}, \cite{decentalized}. FedAvg is the most popular FL algorithm, where model weights are uploaded and aggregated at the server. However, subsequent studies revealed that privacy leakage still happens when malicious attackers have access to the model weights, i.e., gradient inversion attack \cite{Hatamizadeh_2022_CVPR} and membership inference attack \cite{MIA2017}. To further enhance the privacy protection of FL systems, FL with Differential Privacy (DP) guarantees, DPFL in short, is proposed \cite{brendan2018learning,FLDP2020,pmlr-v151-noble22a}. DP can quantify and track the privacy loss of the system. The Gaussian mechanism is mostly adopted to provide DP, where each client adds Gaussian noise to the model weights before uploading them to the server for aggregation.

Large communication costs in transmitting the weights are another major challenge for FL. An effective method to reduce the communication overhead is quantizing the weights \cite{reisizadeh2020fedpaq, LiuTWC2022}. 
There are mainly two methods to achieve the communication-efficient DPFL. One is to add discrete noise to the quantized weight \cite{cpsgd, discreteGaussian,bnn}. The privacy guarantee of such mechanisms can be directly derived by analyzing the discrete noise itself. Another approach is to quantize the weights perturbed by Gaussian noise \cite{FLQuan,wang2021d2pfed}, which is more straightforward to implement than discrete noise and commonly adopted in practice. However, the precise measure of privacy protection becomes difficult when quantization is introduced.

We focus on the second approach, where quantization is applied after adding Gaussian noise to the weights. In this case, understanding the effect of quantization on privacy protection is crucial. Previous works \cite{FLQuan,wang2021d2pfed} investigated this problem in a scalar quantizer setting. They approximated the process as adding Gaussian noise to the weights permuted by quantization noise. As quantization noise increases the sensitivity of the weight, the authors reached a conclusion that quantization damages privacy protection. However, the effect of quantization on the noise scale was not fully considered in the discussion. Given the noise scale and sensitivity both affect privacy protection, it is possible that quantization can benefit privacy protection. This possibility is further supported by the post-processing theorem of DP, which shows that quantized weights should have at least the same level of DP protection as the non-quantized one \cite{dwork2014algorithmic}. 

In this paper, we propose a theoretical analysis to understand the effect of the quantized Gaussian mechanism on privacy protection. 
Instead of approximating the problem with another form as in \cite{FLQuan,wang2021d2pfed}, we directly derive the privacy budget of quantized Gaussian mechanism from the definition of R\'enyi Differential Privacy (RDP). RDP is a variant of DP, and is computationally more convenient. Our analysis demonstrates lower quantization levels lead to better privacy protection. To verify the theoretical analysis, we experimentally evaluate the privacy leakage by the Membership Inference Attacks (MIA), and use the MIA score to track the privacy loss of the FL model with different quantization levels. The experiment results validate the theoretical findings and show that quantization indeed benefits privacy protection.

\section{System Model and Preliminary} \label{preliminary}

In this section, we introduce some preliminary knowledge for FL and MIA, together with the definition of DP and RDP.

\subsection{Federated Learning with Stochastic Quantization} \label{subsec: quantizedGaussian}
The system model of FL with RDP and Stochastic quantization is illustrated in Fig. \ref{fig_system}. We only discuss FL with stochastic quantization here, and leave the illustration of noise adding and RDP in Seciton \ref{subsec: RDP}. First, we denote the set of private datasets of client $i$ as $\cD_i$, and global weight as $\vw$. Let $F_i(\vw)$ be the empirical loss of a model $\vw$ evaluated over the dataset $\cD_i$ and set $\cD=\cup_{i=1}^n \cD_i$. The training goal of FL is  to obtain the optimal weights vector $\vw^{\rm opt}$ satisfying
\begin{equation}\label{eq:w_opt_def}
    \vw^{\rm opt} = \argmin_{\vw} \left\{F(\vw)\triangleq \sum_{i=1}^n \alpha_i F_i\left(\vw\right)\right\},
\end{equation}
where $\alpha_i=\frac{\abs{\cD_i}}{\abs{\cD}}$ is the averaging coefficients.
 Denote $\vw_t$ as the global weight at iteration $t$. With the FedAvg algorithm, each client $i$ updates $\vw_t^i$ based on local dataset $\cD_i$, the server then updates the global weights as 
\begin{equation}
\label{eq:fl_update}
    \vw_{t+1} \triangleq \sum_{i=1}^n \alpha_i \vw^i_{t}.
\end{equation}

To reduce communication costs and utilize a secure communication scheme, the client transmits the quantized version of the weight update $\vw^i_{t} - \vw_t$ to the server. The $k$-level stochastic quantizer considered in this paper \cite{kQuantize} with clipping is shown below.

\begin{definition}[$k$-Level Stochastic Quantizer with clipping] \label{def:kStochasticQuantizer}
Let $\vw$ be an input vector of dimension $d$. Define the clipping threshold $C_q>0$.
We first clip the input as 
$\tilde{\vw} = \vw \cdot \min\{ 1, \frac{C_q}{\|\vw\|}\}.$
The codomain of the $k$-level quantizer is defined as $\mathbb{S}^d$ such that $\mathbb{S}=\{B(r) | B(r)=-C_q+\frac{2C_q}{k-1} r \text{, for } r=0,1, \ldots, k-1\}$. A $k$-level stochastic quantizer with clipping threshold $C_q$, $Q(\vw ; k; C_q)$, randomly maps each clipped element $\tilde{w}_i$ to one of the two adjacent quantization levels as
\begin{equation}
    \{Q(\vw ; k,C_q)\}_i= \begin{cases}B(r), & \text { w.p. } \frac{B(r+1)-\tilde{w}_i}{B(r+1)-B(r)} \\ B(r+1), & \text { w.p. } \frac{\tilde{w}_i-B(r)}{B(r+1)-B(r)}\end{cases}
\end{equation}
where $\{Q(\vw ; k,C_q)\}_i$ is the $i^{\text{th}}$ element of $Q(\vw ; k,C_q)$, and $B(r) \leq \tilde{w}_i \leq B(r+1)$.
\end{definition}

The update of the global weights with quantization can then be reformulated as:
\begin{equation}\label{eq:fl_update_quant}
    \vw_{t+1} \triangleq \vw_t + \sum_{i=1}^n \alpha_i Q(\vw^i_{t} - \vw_t; k, C_q).
\end{equation}

\begin{figure}[tbp]
\centerline{\includegraphics[width=0.5\textwidth]{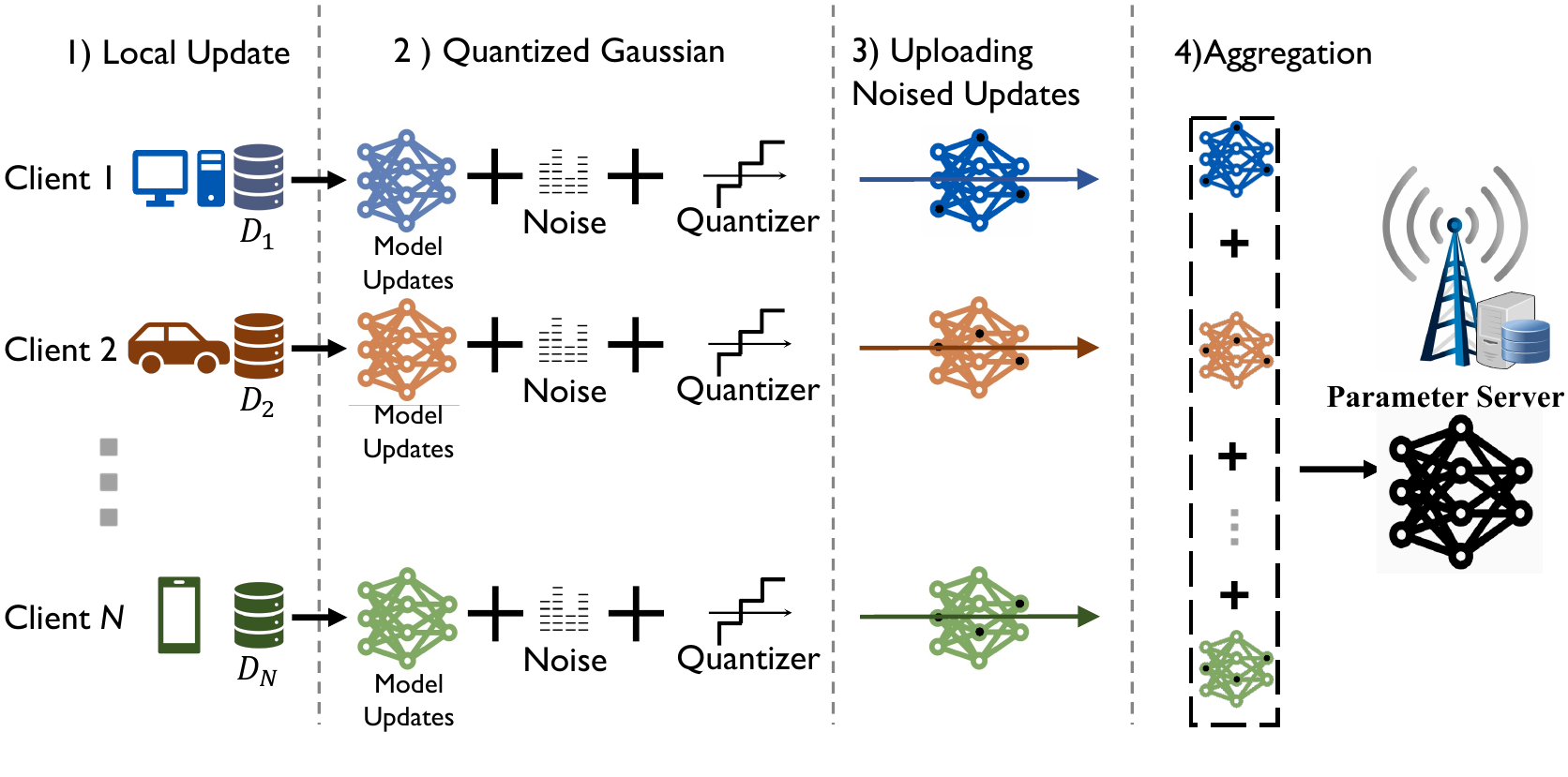}}

\caption{Illustration of Federated Learning with quantized Gaussian mechanism applied. The weight of each client is first perturbed by Gaussian noise and then passed to the quantizer.}
\label{fig_system}
\end{figure}

\subsection{R\'enyi Differential Privacy} \label{subsec: RDP}
RDP proposed in \cite{RenyiDP} can be regarded as a variant of the standard DP. We adopt RDP for our theoretical analysis in Section \ref{theory} mainly because of its computational convenience. Before introducing RDP, we first provide the formal definition of DP \cite{dwork2014algorithmic} as follows. 
\begin{definition} [$(\epsilon, \delta)$-DP]
An algorithm $\mathcal{M}: \mathcal{D}\rightarrow \mathcal{S}$ is $(\epsilon, \delta)$-DP if, for all neighboring databases $D, D' \in \mathcal{D}$ and all $S \subseteq \mathcal{S}$,
\begin{equation}
    \Pr[\mathcal{M}(D)\in S] \leq e^\epsilon \Pr[\mathcal{M}(D')\in S] +\delta,
\end{equation}
where $\epsilon \geq 0$ and $0<\delta<1$.
\end{definition}

The parameter $\epsilon$ is usually referred to as the privacy budget. Next, we define RDP and $\ell_2$-sensitivity. $\ell_2$-sensitivity is an important concept in DP and can be useful in the privacy budget calculation that follows.
\begin{definition}[$(\alpha, \epsilon)$-RDP and $\ell_2$-sensitivity]\label{def:rdp}
A randomized mechanism $f\colon \mathcal{D}\mapsto\mathcal{R}$ is said to have $(\alpha, \epsilon)$-RDP, if for any neighboring $D,D'\in\mathcal{D}$ it holds that 
\begin{equation}
    D_\alpha\left(f(D)\|f(D')\right)\leq\epsilon,
\end{equation}
where for $\alpha > 1$
\begin{equation}
    D_\alpha(P\|Q)\triangleq\frac{1}{\alpha-1}\log \mathbb{E}_{x\sim Q}\left( \frac{P(x)}{Q(x)}\right)^\alpha,
\end{equation}
and for $\alpha=1$
\begin{equation}
    D_1(P\|Q)\triangleq \log \mathbb{E}_{x\sim P} \frac{P(x)}{Q(x)}.
\end{equation}
The $\ell_2$ sensitivity of $f$, denoted by $\Delta_2$ can be defined as 
\begin{equation}
    \Delta_2 = \max_{D,D'} \|f(D)-f(D')\|_2.
\end{equation}
\end{definition}



When $\alpha \rightarrow \infty$, $(\alpha, \epsilon)$-RDP is equivalent to standard $(\epsilon, 0)$-DP. In this work, we adopt the client-level DP, where datasets $D$ and $D'$ are considered neighboring if one is created by either removing or replacing all the data from a single client. 
Client-level DP guarantees that all data associated with the client are protected, and is stronger than item-level protection. It is particularly suited for FL settings that involve many clients, each with a relatively small dataset. We now state one other important lemma of RDP that is useful for our following analysis.

\begin{lemma}[monotonicity]\label{lema:monotonicity}
 $D_{\alpha}$ monotonically increases with $\alpha$. i.e., $D_{\alpha_1}(P\|Q) \leq D_{\alpha_2}(P\|Q)$ for $1\leq\alpha_1\leq\alpha_2\leq\infty$.
\end{lemma}

From this lemma, we can get that 
    $(\alpha_2, \epsilon)$-RDP implies $(\alpha_1, \epsilon)$-RDP for $\alpha_2>\alpha_1\geq1$.
    
\subsection{Likelihood Ratio Membership Inference Attack}


The task of MIA \cite{shokri2017membership} is to infer if a specific data sample $x$ was one member of the training set. It is the simplest and most widely deployed attack method for auditing privacy leakage. 

Likelihood ratio (LiRA) membership inference attack proposed in \cite{lora} is one state-of-the-art MIA method.
LiRA makes the membership inference by examining the model loss $\mathcal{L}(x; f)$. The logic is that models usually have a lower loss on their training samples. One important assumption of LiRA is that the loss for training and non-training samples both follow Gaussian distributions, denoted as $N_\text{in}(\mu_\text{in}, \sigma^2_\text{in})$ and $N_\text{out}(\mu_\text{out}, \sigma^2_\text{out})$, respectively. This parametric modeling makes LiRA more stable compared to other machine learning based MIAs. Estimating the parameters of these two distributions is crucial for LiRA to make membership inferences. 
 
The attack process of LiRA is as follows. The attacker first trains a set of shadow models with architecture similar to the target model.
Half of the shadow models have $x$ in their training set, denoted as set $\cup_i\{f^{\text{in}_i}\}$, while other half not having $x$ in their training set is denoted as set $\cup_i\{f^{\text{out}_i}\}$. 
$\mu_\text{in}$, $\sigma_\text{in}$ are estimated as mean and variance of $\cup_i\{\mathcal{L}(x; f^{\text{in}_i})\}$, and $\mu_\text{out}$, $\sigma_\text{out}$ are estimated via $\cup_i\{\mathcal{L}(x; f^{\text{out}_i})\}$.
Finally, to infer membership of $x$ for the target model $f$, LiRA computes 
$l = \mathcal{L}(x, f)$ and then compares $\Pr\left[l| N_\text{in}\right]$ and $\Pr\left[l | N_\text{out}\right]$. If $\Pr\left[l| N_\text{in}\right] > \Pr\left[l | N_\text{out}\right]$, the sample $x$ is inferred as a member of the training set of the target model $f$. 

LiRA requests training a set of shadow models for each sample $x$ they are to infer, which can be computationally very expensive. One offline variant of LiRA is to train a set of shadow models beforehand and include no samples we are to infer during shadow model training. The workflow of the offline version of LiRA is illustrated in Fig. \ref{fig_system_lira}. This set of shadow models is qualified to serve as $\cup_i\{f^{\text{out}_i}\}$ for all samples that we need to infer membership. The membership inference of each $x$ is then fully based on $\Pr\left[l | N_\text{out}\right]$.
To save inference time, the offline version of LiRA is adopted in this work. 

\begin{figure}[tbp]
\centerline{\includegraphics[width=0.5\textwidth]{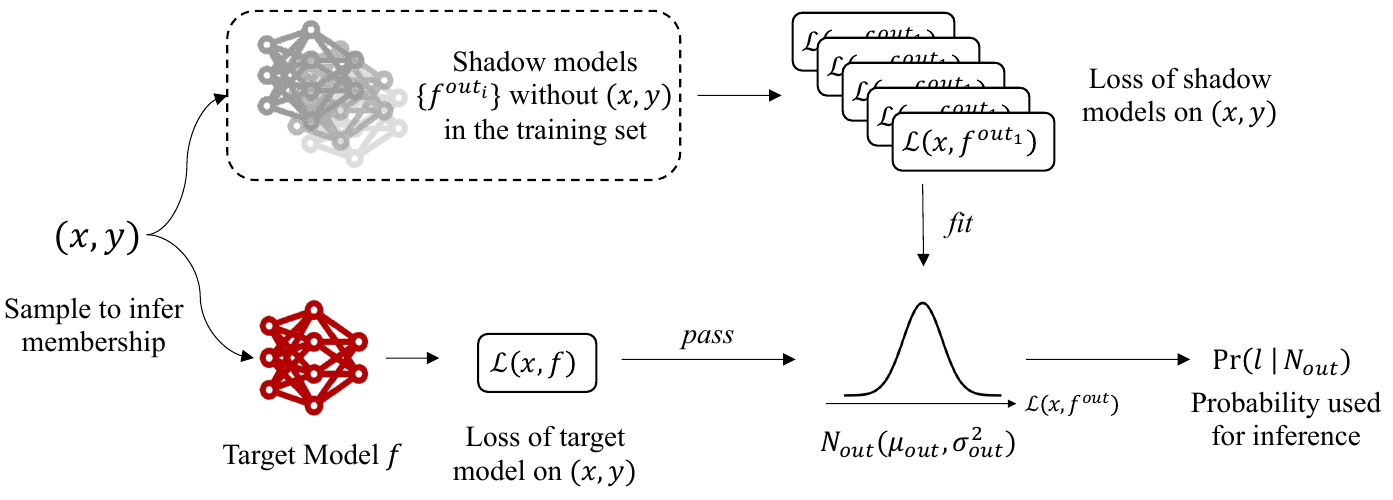}}

\caption{Working flow of the offline variant of LiRA. The attacker trains a set of shadow models that do not contain $x$ in their training set, and utilizes them to estimate $\mu_\text{out}$, $\sigma_\text{out}$. The attacker uses $\Pr\left[l | N_\text{out}\right]$ to infer membership of $x$.}
\label{fig_system_lira}
\end{figure}


\section{Algorithm and Theoretical Analysis} \label{theory}

\subsection{FedAvg with Stochastic Quantization and RDP}
The FedAvg with Stochastic Quantization and RDP is illustrated in Alg. \ref{alg:FL Quantize and RDP}. The noise adding and quantization procedure in lines 8 and 9 demonstrate how quantized Gaussian noise is applied. We assume that perfect encoding and decoding can be achieved during transmission.

In the outlined FL framework, each client first updates its local weights before computing the difference from the global model. To limit the $\ell_2$-sensitivity of this weight discrepancy to $C_q$, it is clipped to $\frac{C_q}{2}$ as specified in line 7. Subsequently, Gaussian noise $N(0, \sigma^2)$ is added to the clipped weight difference to fulfill the privacy preservation criteria. Next, the perturbed weight difference is passed to stochastic k-level quantizer in Def. \ref{def:kStochasticQuantizer} with parameter $C_q$. This quantized difference is then encoded and transmitted to the server for the aggregation process.

\begin{algorithm}[bthp]
\caption{FedAvg with Stochastic Quantization and RDP} \label{alg:FL Quantize and RDP}
\begin{algorithmic}[1]
\STATE{\bfseries Input:} $N$ local devices, each local device dataset $D_i \in \mathcal{D}$ ($i=1, \cdots, N$), stochastic quantizer parameters $(k, C_q)$
\FOR{$t = 0, \cdots, T-1$}
\STATE Server broadcasts $\vw_t$ to $n$ sampled local devices from total $N$ local devices;
\FOR{each local device $i$ in parallel}
\STATE Update the weight $\vw_t$ to $\vw^i_t$
\STATE Calculate the weight difference $\vh^i_t = \vw^i_{t} - \vw_t$
\STATE Clip the difference $$\widetilde{\vh^i_t} = \vh^i_t \cdot \min\{ 1, \frac{C_q}{2\|\{{\vh}^i_t\|}\}$$
\STATE Add noise: $\overline{\vh^i_t} \leftarrow \widetilde{\vh^i_t} + N(0, \sigma^2)$
\STATE Quantize the model difference
$$\overline{\overline{{\vh}^i_t}} = Q(\overline{\vh^i_t} ; k,C_q)$$
\STATE Encode and upload $\overline{\overline{{\vh}^i_t}}$ to the server.
\ENDFOR
\STATE Server decodes $\overline{\overline{{\vh}^i_t}}$
\STATE Server update weights: $\vw_{t+1} \triangleq \vw_t + \sum_{i=1}^n \alpha_i \overline{\overline{{\vh}^i_t}}$
\ENDFOR
\end{algorithmic}
\end{algorithm}

\subsection{Theoretical Analysis} \label{subsec:theory analysis}

In this subsection, we perform the privacy analysis on the quantized Gaussian mechanism with RDP to demonstrate the exact relationship between the quantization level and privacy leakage. First, we present the post-processing theorem to suggest why quantization will not damage privacy protection. 

\begin{theorem}[Post-processing \cite{dwork2014algorithmic}] \label{thm:post-processing}
Let  $f\colon \mathcal{D}\mapsto\mathcal{R}$ be a randomized algorithm that is $(\alpha, \epsilon)$-RDP, let $g: \mathcal{R} \rightarrow \mathcal{R}^{\prime}$ be a randomized mapping, then $g(f(\cdot))$ is also $(\alpha, \epsilon)$-RDP.
\end{theorem}

Following the above post-processing theorem, if the Gaussian noise in Alg. \ref{alg:FL Quantize and RDP} ensures $(\alpha, \epsilon)$-RDP, the quantized Gaussian will also ensure $(\alpha, \epsilon)$-RDP. 
Next, we derive the privacy budget $\epsilon$ of the quantized Gaussian mechanism directly from the definition of RDP and show that a lower quantization level leads to better privacy protection. 

Mathematically, the output of the quantized Gaussian mechanism applied on a scalar $x$ can be represented as
\begin{equation}
    {G}_{\sigma, k, C_q} x=
    Q\left(x+N(0,\sigma^2), k, C_q\right),
\end{equation}
where $N(0,\sigma^2)$ denotes the Gaussian distribution with standard deviation $\sigma$ and mean 0, and $Q(x, k, C_q)$ is the k-level stochastic quantizer introduced in Def.~\ref{def:kStochasticQuantizer}.

Following Def.~\ref{def:rdp}, to calculate the privacy budget $\epsilon$ for quantized Gaussian mechanism with $\ell_2$-sensitivity $\Delta_2$ under $\alpha$, it is essential to first obtain the probability mass function (pmf) of the output of quantized Gaussian. Represent the pmf of the output of quantized Gaussian on $x$ and $P_x$. $\epsilon$ can then be calculated as the supremum of 
$D_{\alpha}(P_x\|P_{x'})$, where $\|x-x'\|_2 \leq \Delta_2$. The following lemma provides the pmf of the quantized Gaussian.

\begin{lemma} \label{lemma: pdfOFquantizedGAUSSIAN}
    For quantized Gaussian parameterized by $(\epsilon, k, C_q)$, define the distance between two quantization levels as $\delta = \frac{2C_q}{k-1}$.
    Further define $B(r)=-C_q+\frac{2C_q}{k-1} r \text{, for } r=0,1, \ldots, k-1$.
    Assume $x\in\mathbb{R}$ and $ x \in [-\frac{C_q}{2}, \frac{C_q}{2}]$, then
    \begin{equation}
    \begin{split}
        P_x & \left[B(k-1)\right] \\
        & = \int_{B(k-2)}^{B(k-1)}f(x)\frac{x-B(k-2)}{\delta}dx + \int_{B(k-1)}^{\infty}f(x)dx\\
        P_x & \left[B(0)\right]\\
        & = \int_{-\infty}^{B(0)}f(x)dx + \int_{B(0)}^{B(1)}f(x)\frac{B(1)-x}{\delta}dx,\\
    \end{split}
    \end{equation}
    when $1 \leq r < k-1$, we have:
    \begin{multline}
        P_x  \left[  B(r)\right]
         = \int_{B(r-1)}^{B(r)}f(x)\frac{x-B(r-1)}{\delta}dx \\ + \int_{B(r)}^{B(r+1)}f(x)\frac{B(r+1)-x}{\delta}dx,\\
    \end{multline}
    where f(x) is the probability density function of the normal distribution $N(x,\sigma^2)$.
\end{lemma}

With the derived pmf of quantized Gaussian, we obtain the RDP of the quantized Gaussian mechanism in the following theorem.

\begin{theorem} \label{thm:bound of epsilon}
    Let $k, C_q, \delta$ be parameters in Lemma \ref{lemma: pdfOFquantizedGAUSSIAN}. For scalar $x$ and $x'$ in $\left[-\frac{C_q}{2}, \frac{C_q}{2}\right]$, represent the pmf of the output of quantized Gaussian on $x$ and $x'$ as $P_x$ and $P_{x'}$, respectively. Then for the quantized Gaussian mechanism, we have 
    \begin{equation}\label{eq:epsilon1}
    D_1 (P_{x}||P_{{x'}}) 
        \leq D_{KL} (P_{\frac{C_q}{2}}||P_{-\frac{C_q}{2}})
    \end{equation}
    and 
    \begin{equation}
    \begin{aligned}
        D_{\alpha} (P_{x}||P_{{x'}})& \leq D_{\infty} (P_{x}||P_{{x'}})\\
        & \leq
        \log \Big(\frac{{\delta}}{\int_{B(k-2)}^{B(k-1)}f(x)\left({x-B(k-2)}\right)dx} \Big),
    \end{aligned}
    \end{equation}
    where f is the probability density function of the normal distribution $N(-\frac{C_q}{2},\sigma^2)$, and $D_{KL}$ represents the Kullback-Leibler divergence.
\end{theorem}

\begin{proof}
    When $\alpha=1$, we have
    \begin{equation}
    \begin{aligned}
        D_1 (P_{x}||P_{{x'}}) 
        & \leq 
        D_1 (P_{\frac{C_q}{2}}||P_{-\frac{C_q}{2}})\\
        & = D_{KL} (P_{\frac{C_q}{2}}||P_{-\frac{C_q}{2}}),\\
    \end{aligned}
    \end{equation}
    For $\alpha >1$, we have
    \begin{equation}
    \begin{aligned}
         D_\alpha  & (P_{x}||P_{{x'}}) \\
        \,&\leq   D_\infty (P_{x}||P_{{x'}}) \\
        \,&=  \sup_{r \in \{0, \cdots, k-1\}} \log \Big( \frac{P_{x}[B(r)]}{P_{{x'}}[B(r)]} \Big)\\
        \,&\leq  \log \Big(\frac{1}{\min\{P_{-\frac{C_q}{2}}[B(k-2)], P_{-\frac{C_q}{2}} [B(k-1)]\}} \Big) \\
        & < \log \Big(\frac{{\delta}}{\int_{B(k-2)}^{B(k-1)}f(x)\left({x-B(k-2)}\right)dx} \Big), \\
    \end{aligned}
    \end{equation}
    where the first inequality follows the monotonicity of the R\'enyi divergence presented in Lemma \ref{lema:monotonicity}.
\end{proof}

Our analysis focuses on cases where $\alpha=1$ and $\alpha \rightarrow \infty$. Denote the privacy budget in these two scenarios as $\epsilon_1$ and $\epsilon_{\infty}$, respectively. Both $\epsilon_1$ and $\epsilon_{\infty}$ can be computed numerically given clipping parameter $C_q$ and quantization level $k$. With the monotonicity of R\'enyi Divergences, $\epsilon_1$ and $\epsilon_{\infty}$ are the lower and upper bounds of the privacy budget, $\epsilon$, for any $\alpha$ within $[1, \infty]$. 

In section \ref{subsec:theoryExperiment}, we experimentally demonstrate that $\epsilon_1$ increases with quantization level $k$ given clipping parameter $C_q$, which means that a lower quantization bit leads to better privacy protection. We can also show that $\epsilon_1$ is tighter than the privacy budget of the Gaussian mechanism when $\alpha=1$, indicating that quantization indeed benefit privacy. When $\alpha \rightarrow \infty$, $(\alpha, \epsilon_{\infty})$-RDP is equivalent to $(\epsilon_{\infty}, 0)$-DP as in the traditional definition of DP. We can further prove, instead of numerically showing, that $\epsilon_{\infty}$ increases monotonically with the quantization level $k$ given $C_q$ in the Lemma \ref{lemma:epsilonInfinity} that follows. 

\begin{lemma}\label{lemma:epsilonInfinity}
    $\epsilon_{\infty}$ increases monotonically with $k$ given $C_q$
\end{lemma}

\begin{proof}
To prove $\epsilon_{\infty}$ increases with $k$ given $C_q$, it is sufficient to show that $\epsilon_{\infty}$ decreases with $\delta$, where $\delta = \frac{2C_q}{k-1}$.
    First define 
    $g(\delta) = {\int_{C_q-\delta}^{C_q}f(x)\left({\frac{x-(C_q-\delta)}{\delta}}\right)dx},$
    then $\epsilon_{\infty} = \log(\frac{1}{g(\delta)})$. Set $u=\frac{x-(C_q-\delta)}{\delta}$, then we have
    \begin{equation} \label{equ:mono}
        \begin{aligned}
        \frac{\partial g}{\partial \delta} 
        = & \frac{\partial}{\partial \delta}\left(\delta u\int_{0}^{1}f(x) du\right) \\
        = & u\int_{0}^{1}f(x)du + \delta u\int_{0}^{1}\frac{\partial f(x)}{\partial x}\frac{\partial x}{\partial \delta}du \\
        = & u\int_{0}^{1}f(x)du + \delta u(1-u) \int_{0}^{1}\frac{x+\frac{C_q}{2}}{\sigma^2}f(x)du \\
        = & uf(x) \int_{0}^{1} \left(1+(1-u)\frac{x+\frac{C_q}{2}}{\sigma^2}\right)du.
        \end{aligned}
    \end{equation}
    As $x \in [-\frac{C_q}{2}, \frac{C_q}{2}]$ from Lemma \ref{lemma: pdfOFquantizedGAUSSIAN}. The term in the last integral of \eqref{equ:mono} is always larger than 0. Therefore $\frac{\partial g}{\partial \delta} > 0$, $\frac{\partial \epsilon_{\infty}}{\partial \delta} < 0$. As $\delta = \frac{2C_q}{(k-1)}$, we conclude that $\epsilon_{\infty}$ decrease with $k$ with fixed $C_q$.
\end{proof}

\section{Experiment Results} \label{experiment}

In this section, we conduct experiments to validate the theoretical analysis. We first show that quantized Gaussian has a tighter privacy budget compared to the Gaussian mechanism. Next, we expand our experiments beyond the privacy guarantees of the quantized Gaussian mechanism itself, and study how quantized Gaussian integrates with the FL framework described in Section \ref{subsec: quantizedGaussian}. Specifically, we employ MIA to empirically verify the correlation between quantization level and privacy leakage in the FL setting.

\subsection{Numerical Evaluation of the Privacy Budget} \label{subsec:theoryExperiment}
In Section \ref{subsec:theory analysis}, we derive the privacy budget $\epsilon_1$ and $\epsilon_{\infty}$ of quantized Gaussian for $\alpha = 1$ and $\alpha \rightarrow \infty$.
Next, we numerically compute and compare the privacy budgets of quantized Gaussian and Gaussian. 
Following Theorem \ref{thm:post-processing}, $\epsilon_1$ and $\epsilon_{\infty}$ should be as tight as the privacy budget of Gaussian. 
Referring to \cite{RenyiDP}, with the $\ell_2$-sensitivity set to $C_q$, the Gaussian mechanism garantees to $(\alpha, \frac{\alpha {C_q}^2}{2\sigma^2})$-Renyi Differential Privacy (RDP), where $\sigma$ represents the noise variance. 

In the limit as $\alpha \rightarrow \infty$, the term $\frac{\alpha {C_q}^2}{2\sigma^2}$ approaches infinity, whereas $\epsilon_{\infty}$ for the quantized Gaussian remains finite. For a fair and meaningful comparison, we compare the budget of two mechanisms when $\alpha = 1$. We change the quantization levels $k$ while keeping $\ell_2$-sensitivity $\Delta_2= 1$ and noise variace $\sigma = 1$. The relationship between the privacy budget and quantization levels is depicted in Fig. \ref{fig_bdgetplot}.


\begin{figure}[b]
\centerline{\includegraphics[width=0.4\textwidth]{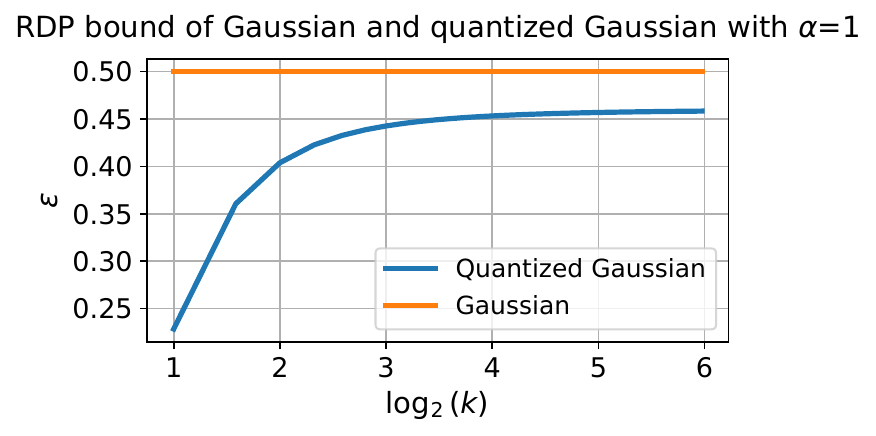}}
\caption{Comparison of RDP Privacy budgets for Gaussian and quantized Gaussian mechanisms at $\alpha=1$, with varying quantization levels $k$ from 2 to 64 while keeping $\sigma=1$ and $\Delta_2=1$. $x$-axis depicted in a logarithmic scale represents the quantization level $k$.}
\label{fig_bdgetplot}
\end{figure}

From Fig. \ref{fig_bdgetplot}, we observe that the privacy budget of quantized Gaussian, illustrated by the blue curve, increases monotonically with quantization level $k$. This confirms that a lower quantization level offers better privacy protection. Furthermore, at all quantization levels for $\alpha=1$, the quantized Gaussian mechanism demonstrates a more stringent privacy budget compared to the Gaussian mechanism, indicating the benefits of quantization for privacy protection.


\subsection{FedAvg with quantized Gaussian and MIA setting}

This subsection introduces the setup for the experiment results in section \ref{subsec: experiment result}. We start by training target models within an FL framework, following Algorithm \ref{alg:FL Quantize and RDP} with various quantization levels and privacy budgets, and compare their actual privacy leakage. The parameters guiding the training of target models are concisely presented in Table \ref{tab:experiment_parameters}.

\begin{table}[tbp]
\caption{Parameters for FL Experiments}
\label{tab:experiment_parameters}
\centering
\begin{tabular}{ll}
\hline
Parameter & Value \\
\hline
Architecture & ResNet-18 \\
Dataset & Cifar-10 \\
Client Number ($N$) & 100 \\
Communication Rounds ($T$) & 150 \\
Total training Samples & 45,700 (i.i.d. among clients) \\
Client Optimizer & Adam \cite{Adam} \\
Batch Size & 128 \\
Learning Rate ($\alpha$) & 0.001 \\
Adam $\beta_1$ & 0.9 \\
Adam $\beta_2$ & 0.999 \\
Clipping Threshold ($C_q$) & 1 \\
\hline
\end{tabular}
\end{table}

To evaluate privacy leakage, we utilize MIA accuracy as the metric. The LiRA method, recognized as the state-of-the-art for MIAs, is employed following \cite{lora}. In our setting, the attacker trains 64 shadow models centrally, each leveraging 20,000 data samples in a centralized manner. Given the constrained dataset size, overlaps between the training datasets of shadow models and the target model are inevitable. However, as justified in \cite{lora}, such overlaps do not compromise the validity of the attack outcomes.

\subsection{MIA accuracy on FedAvg with quantized Gaussian} \label{subsec: experiment result}

In this subsection, we attack FedAvg with quantized Gaussian by MIA, and evaluate the privacy leakage of the model under different quantization levels for a given noise scale injected during training. Specifically, We inject Gaussian noise with a variance of 0.03 in the training process, ensuring $(5, 10^{-5})$ client-level Differential Privacy (DP) for the Gaussian mechanism, as depicted in Fig. \ref{fig_miascore}.

\begin{figure}[tbp]
\centerline{\includegraphics[width=0.33\textwidth]{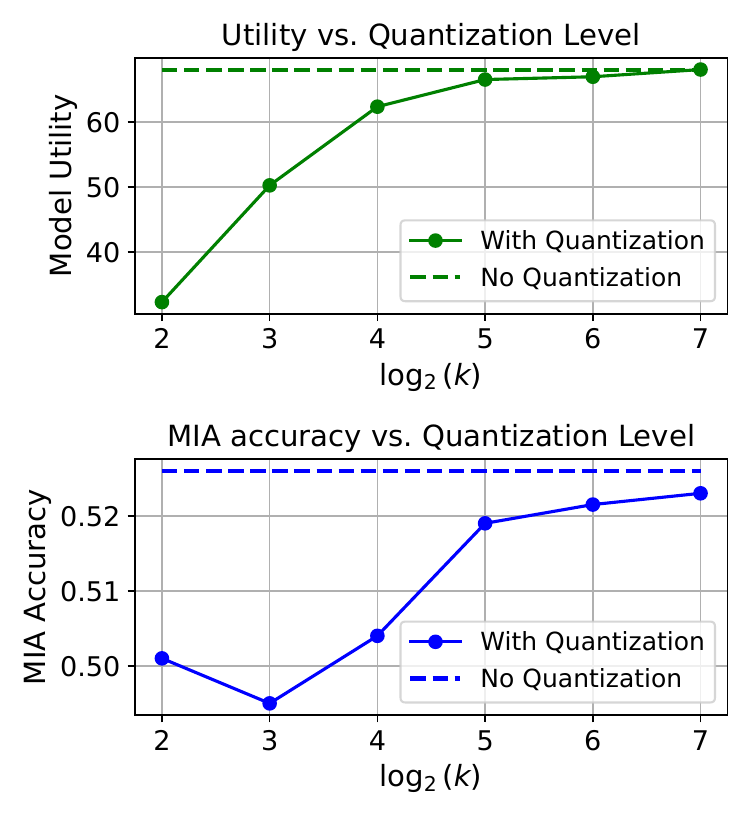}}

\caption{Comparison of utility and MIA accuracy among models employing Gaussian and quantized Gaussian mechanism, all with a noise variance of 0.03. The dashed line represents the model without quantization. $x$-axis depicted in a logarithmic scale represents the quantization level $k$.}

\label{fig_miascore}
\end{figure}

Comparing the utility of FedAvg with quantized Gaussian under different quantization levels,  it is apparent that a lower quantization leads to a lower utility of the model. Here, utility is defined as the accuracy of the model in classifying CIFAR-10 dataset images. Subsequently, we compare the MIA accuracy on models with different quantization levels. The MIA accuracy on the quantized model is always lower than the unquantized one, suggesting that quantization benefits privacy protection. Furthermore, a positive correlation between quantization level and MIA accuracy can also be observed; a larger quantization level leads to a higher MIA accuracy, i.e., worse privacy protection. This observation aligns with Theorem \ref{thm:bound of epsilon} derived in Section \ref{subsec:theory analysis}. 

We further explore the utility and MIA accuracy of FedAvg across different noise scales, with results presented in Fig. \ref{fig_miascore2} for models with quantization levels $k=16$, $k=32$, and without quantization. The $x$-axis, denoted as $\sigma$, represents the variance of Gaussian noise during training. Notably, $\sigma=0.0864$ and $\sigma=0.0164$ ensure $(1.5, 10^{-5})$ and $(10, 10^{-5})$ client-level DP, respectively, for the Gaussian mechanism. An almost monotonic decrease in MIA accuracy with increasing noise scale is observed across all scenarios, underscoring the efficacy of this metric as an indicator of privacy protection. Furthermore quantized model consistently exhibits a lower MIA accuracy compared to the unquantized one, and a lower quantization level corresponds to lower MIA accuracy across nearly all noise scales. This not only aligns with the insights drawn from Fig. \ref{fig_miascore} but also affirms our argument that quantization indeed protects privacy.

\begin{figure}[tbp]
\centerline{\includegraphics[width=0.33\textwidth]{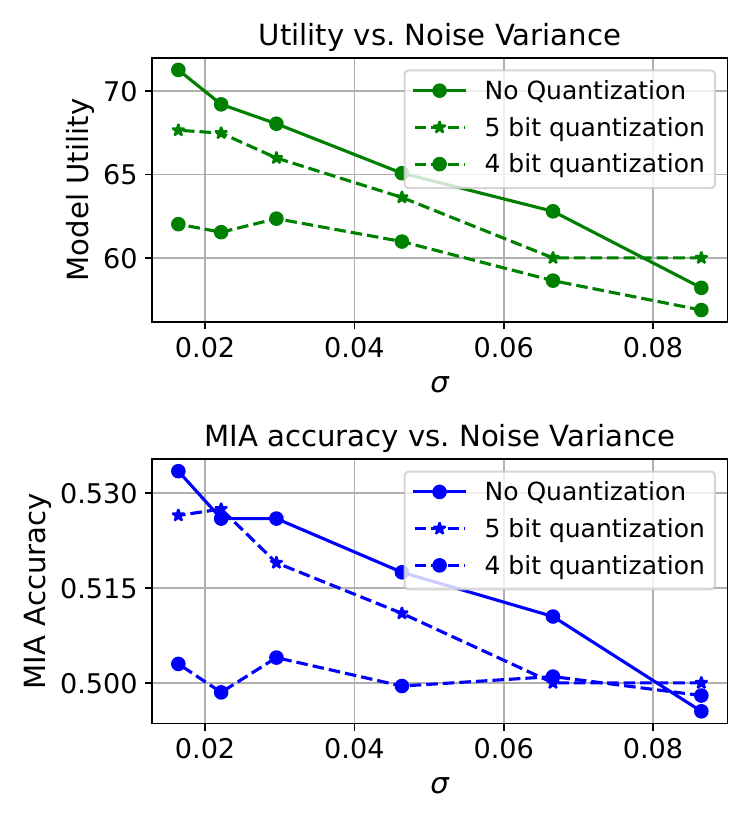}}

\caption{Utility and MIA accuracy comparison among models employing Gaussian mechanism and quantized Gaussian mechanism with quantization level $k=16$ and $k=32$, varying the Gaussian noise variance during training.}

\label{fig_miascore2}
\end{figure}

\section{Conclusions} \label{conclusion}
In this work, we provided a theoretical analysis of the privacy budget of quantized Gaussian mechanisms, demonstrating that lower quantization levels lead to better privacy protection. Numerical results validated the theoretical results and show that quantization effectively enhances privacy protection. 
Our work provides a new perspective on the correlation between privacy and communication in FL, highlighting the benefits of quantization in protecting privacy.

\bibliographystyle{IEEEtran}
\bibliography{IEEEabrv,egbib_updated}

\end{document}